\documentclass[lettersize,journal]{IEEEtran}
\usepackage{amsmath,amssymb,amsthm,amsfonts}
\usepackage[ruled]{algorithm2e}
\usepackage{array}
\usepackage[caption=false,font=normalsize,labelfont=sf,textfont=sf]{subfig}
\usepackage{textcomp}
\usepackage{stfloats}
\usepackage{url}
\usepackage{verbatim}
\usepackage{graphicx}
\usepackage{cite}
\usepackage[hidelinks]{hyperref}

\newtheorem{proposition}{Proposition}
\newtheorem{lemma}[proposition]{Lemma}
\newtheorem{theorem}[proposition]{Theorem}

\DeclareMathOperator{\Tr}{Tr}
\DeclareMathOperator{\Log}{Log}

\begin{document}
\title{Viking: Variational Bayesian Variance Tracking}

\author{Joseph {de Vilmarest} and Olivier Wintenberger\thanks{J. de Vilmarest (joseph.de\_vilmarest@sorbonne-universite.fr) and O. Wintenberger (olivier.wintenberger@sorbonne-universite.fr) are with the Laboratoire de Probabilit\'es, Statistique et Mod\'elisation, Sorbonne Universit\'e, CNRS.} \thanks{J. de Vilmarest is also affiliated with \'Electricit\'e de France R\&D.}}

\markboth{}%
{Shell \MakeLowercase{\textit{et al.}}: Bare Demo of IEEEtran.cls for IEEE Journals}


\maketitle

\begin{abstract}
We consider the problem of time series forecasting in an adaptive setting. We focus on the inference of state-space models under unknown and potentially time-varying noise variances.  We introduce an augmented model in which the variances are represented as auxiliary gaussian latent variables in a tracking mode. As variances are nonnegative, a transformation is chosen and applied to these latent variables. The inference relies on the online variational Bayesian methodology, which consists in minimizing a Kullback-Leibler divergence at each time step. We observe that the minimum of the Kullback-Leibler divergence is an extension of the Kalman filter taking into account the variance uncertainty.  We design a novel algorithm, named Viking, using these optimal recursive updates. For auxiliary latent variables, we use second-order bounds whose optimum admit closed-form solutions.
Experiments on synthetic data show that Viking behaves well and is robust to misspecification.
\end{abstract}


\section{Introduction}
\IEEEPARstart{L}{inear} state-space models have been widely used to model a time series as a gaussian random variable whose mean is a time-varying linear function of covariates. The linear parameter is a latent variable called state, and the hyperparameters of the state-space model are the  covariance matrices of the state and space noises. When these variances are known, the recursive estimation is realized by Kalman filtering \cite{kalman1961new}.

However the state and space noise variances are unknown in most practical applications. A wide literature has emerged to choose them. The estimation of unknown fixed variances on a historical data set is generally realized maximizing the likelihood (see for instance \cite{brockwell1991time,durbin2012time}). Another approach is to estimate these variances (fixed or not) in an online fashion, that is adaptive filtering \cite{mehra1972approaches}.

Recently, recursive variational Bayesian (VB) methods as introduced in \cite{beal2003variational,vsmidl2006variational} have gathered attention in the Kalman filtering community. The objective is the online estimation of potentially time-variant parameters. The difference with the classical Bayesian method is that an approximation is realized at each step in order to make the inference tractable: the distribution of the parameters is estimated by simple factorized distributions. The best factorized distribution is defined as the one minimizing its Kullback-Leibler divergence with the posterior.

A VB approach was first applied to estimate the observation noise covariance matrix in a Kalman filter \cite{sarkka2009recursive}, then extended in \cite{agamennoni2012approximate} to be robust to non-gaussian noise and in \cite{sarkka2013non} to nonlinear state-space models. The covariance matrix is assumed diagonal and the prior used is a product of inverse gamma distributions. To allow for a dynamical noise variance the author use a forgetting factor, multiplying the variances of the inverse gamma posterior distributions by a constant. The method was extended with an inverse Wishart prior \cite{huang2017novel}. At the same time the authors apply the VB approach to correct the covariance matrix of the state after applying Kalman recursions with an inaccurate state noise covariance matrix. The inverse Wishart distribution appears as a nice conjugate prior to generalize the inverse gamma distribution. More recently another adaptive Kalman filter was proposed in \cite{huang2020slide} to estimate simultaneously the state and space noise covariance matrices. The method uses Kalman filtering and smoothing on a slide window and could be described as an online Expectation-Maximization algorithm. In all these methods the dynamics of the variances is introduced through a forgetting factor.

Up to our knowledge, to deal with unknown covariance matrices in state-space models all existing methods apply at each step the standard Kalman filter with an estimate of the variances updated in an adaptive fashion. We claim that it is suboptimal and that the recursive update of the state estimates should leverage the variance uncertainty. In this article we treat the variances as auxiliary latent variables yielding an important degree of freedom in an augmented latent representation. We apply the VB approach and we rely on second-order upper-bounds to tackle the intractability of the VB step.

\subsection{Overview}
We present in Section \ref{sec:variance_tracking} the state-space inference problem, we introduce our augmented dynamical model and the VB principle. As the minimization problem derived in the VB approach doesn't admit a closed-form solution, we derive in Section \ref{sec:approx_vb} an approximation. The algorithm is detailed in Section \ref{sec:algorithm}, and we provide experimental results in Section \ref{sec:experiments}.

\subsection{Notations}
Besides canonical notations we define:
\begin{itemize}
\item
$\mathcal{N}(x\mid \mu,\Sigma)$ the probability density function at point $x$ of the distribution $\mathcal{N}(\mu,\Sigma)$.
\item
For any distribution $p$ and function $\Phi$, $\mathbb{E}_{x\sim p}[\Phi(x)]$ is defined as $\int p(x)\Phi(x)dx$.
\item
For any matrix $M$, $\Delta_{M}$ is the vector composed of the diagonal coefficients of $M$. Reciprocally, for any vector $v$, $D_{v}$ is the diagonal matrix whose diagonal is composed of the coefficients of $v$.
\item
If $\phi:\mathbb{R}\rightarrow\mathbb{R}$ and $x\in\mathbb{R}^d$, $\phi(x)$ is the $d$-dimensional vector obtained by applying $\phi$ to each coordinate of $x$.
\end{itemize}

\section{Variance Tracking}\label{sec:variance_tracking}
We consider the problem of time series forecasting in the univariate setting for simplicity. At each time $t$ we aim at forecasting $y_t\in\mathbb{R}$. To that end we have access to covariates $x_t\in\mathbb{R}^d$ as well as the past observations $x_1,y_1,\hdots x_{t-1},y_{t-1}$.
We focus on a state-space representation where $y_t$ is modelled as a linear function of $x_t$ whose linear parameter evolves dynamically:
\begin{align*}
    & \theta_t = K \theta_{t-1} + \eta_t \,, \\
	& y_t = \theta_t^\top x_t + \varepsilon_t  \,,
\end{align*}
where $\eta_t\sim\mathcal{N}(0,Q_t)$ and $\varepsilon_t\sim\mathcal{N}(0,\sigma_t^2)$ are the state and space noises, and the state follows the initial distribution $\theta_0\sim \mathcal{N}(\hat\theta_0,P_0)$. When $\sigma_t^2$ and $Q_t$ are known, the state vector $\theta_t$ given the past observations follows a gaussian distribution whose mean and covariance can be estimated recursively by the standard Kalman filter~\cite{kalman1961new}. We focus on the setting where these variances are unknown and need to be estimated jointly with the state.

\subsection{Dynamical Variances}
A way to introduce a dynamical estimation of $\sigma_t^2$ and $Q_t$ is to treat them as latent variables in addition to the state vector.
Gaussian distributions are appealing to model a dynamic latent variable.
Therefore we choose a gaussian prior for the variances as for the state vector.
However a variance is necessarily nonnegative, thus we consider transforms of gaussian distributions. Precisely $\sigma_t^2=\exp(a_t)$ and $Q_t=f(b_t)$, where $a_t,b_t$ follow gaussian distributions.
We detail the choice of $f$ in Section \ref{sec:choice_f} where we define either scalar covariance matrices (proportional to $I$) or diagonal ones.
Note that $b_t$ can be of any dimension, as long as $f(b_t)$ is a $d\times d$ positive semidefinite matrix. Our dynamical model is summarized as follows:
\begin{align*}
	& \theta_0\sim\mathcal{N}(\hat\theta_0,P_0)\,, \quad 
	a_0\sim\mathcal{N}(\hat a_0, s_0)\,,\quad
	b_0\sim\mathcal{N}(\hat b_0, \Sigma_0)\,,\\
	& a_t - a_{t-1} \sim\mathcal{N}(0,\rho_a)\,, \quad
	b_t - b_{t-1} \sim\mathcal{N}(0,\rho_b I)\,, \\
	& \theta_t - K\theta_{t-1} \sim\mathcal{N}(0, f(b_t))\,, \\
	& y_t - \theta_t^\top x_t \sim\mathcal{N}(0, \exp(a_t)) \,.
\end{align*}

In these equations we implicitly assume that we have 
\begin{align*}
	& p(\theta_t,a_t,b_t\mid \theta_{t-1},a_{t-1},b_{t-1}) \\
	& \quad = p(\theta_t\mid \theta_{t-1},b_t)p(a_t\mid a_{t-1})p(b_t\mid b_{t-1}) \,.
\end{align*}

\subsection{Bayesian Approach}
We apply a Bayesian approach in order to estimate jointly the state $\theta_t$ and the latent variables $a_t,b_t$ given the past observations. Note however that the problem at hand is the forecast of $y_t$ thus the latent variable of interest is $\theta_t$. The estimation of $a_t$ is necessary for a probabilistic forecast of $y_t$ since it drives the noise variance. The latent variable $b_t$ is added to open enough flexibility for the estimation of the other variables in a dynamical way.
Formally we introduce the filtration of the past observations $\mathcal{F}_t=\sigma(x_1,y_1,\hdots,x_t,y_t)$. At each iteration $t$, the Bayesian approach consists in a prediction step using the model's assumptions and a filtering step using Bayes' rule:
\begin{align*}
	&\text{Prediction:}\quad && p(\theta_t,a_t,b_t\mid \mathcal{F}_{t-1}) \,, \\
	&\text{Filtering:}\quad && p(\theta_t,a_t,b_t\mid \mathcal{F}_t) \,.
\end{align*}

In the case of known variances resolved by the Kalman filter, the prediction step yields $\hat\theta_{t\mid t-1}$ and $P_{t\mid t-1}$, the expected value and covariance matrix of $\theta_t$ given the filtration $\mathcal{F}_{t-1}$. Furthermore we have $p(\theta_t\mid \mathcal{F}_{t-1})=\mathcal{N}(\theta_t\mid\hat\theta_{t\mid t-1},P_{t\mid t-1})$. Then the filtering step yields $\hat\theta_{t\mid t}$ and $P_{t\mid t}$ such that the posterior distribution is $p(\theta_t\mid \mathcal{F}_t)=\mathcal{N}(\theta_t\mid\hat\theta_{t\mid t},P_{t\mid t})$.

However in our variance tracking model the posterior distribution $p(\cdot\mid \mathcal{F}_t)$ is analytically intractable, thus we estimate it with simple distributions. 

\subsection{Variational Bayesian Approach}
A standard approach, referred to as recursive Variational Bayes (VB), is to approximate recursively the posterior distribution with a factorized distribution where each component is of a simple form~\cite{vsmidl2006variational}. We look for $\hat\theta_{t\mid t},P_{t\mid t}, \hat a_{t\mid t},s_{t\mid t}, \hat b_{t\mid t}, \Sigma_{t\mid t}$ such that the product of gaussian distributions $\mathcal{N}(\hat\theta_{t\mid t},P_{t\mid t}) \mathcal{N}(\hat a_{t\mid t},s_{t\mid t})\mathcal{N}(\hat b_{t\mid t},\Sigma_{t\mid t})$ is the best approximation of the posterior distribution.
The  approximation is quantified by the Kullback-Leibler (KL) divergence:
\begin{align}\label{eq:kl}
	KL\Big(\mathcal{N}(\hat\theta_{t\mid t},P_{t\mid t}) \mathcal{N}(\hat a_{t\mid t},s_{t\mid t})\mathcal{N}(\hat b_{t\mid t},\Sigma_{t\mid t})\ ||\ 
	p(\cdot\mid \mathcal{F}_t)\Big),
\end{align}
where $KL(p\ ||\ q)=\mathbb{E}_{x\sim p(x)}[\log(p(x)/q(x))]$. At each step, the VB approach yields a coupled optimization problem in the three gaussian distributions.

Propagating the factorized approximation
\begin{align*}
	p(\theta_{t-1},a_{t-1},b_{t-1}\mid \mathcal{F}_{t-1}) \approx \mathcal{N}(\theta_{t-1}\mid \hat\theta_{t-1\mid t-1},P_{t-1\mid t-1}) \\
	\mathcal{N}(a_t\mid \hat a_{t\mid t},s_{t\mid t}) \mathcal{N}(b_{t-1}\mid \hat b_{t-1\mid t-1},\Sigma_{t-1\mid t-1}) \,,
\end{align*}
the prediction step becomes:
\begin{align*}
	& p(\theta_t,a_t,b_t\mid \mathcal{F}_{t-1}) \\
	& \approx \int \mathcal{N}(\theta_t-K\theta_{t-1}\mid 0,f(b_t))
	\mathcal{N}(a_t- a_{t-1}\mid 0,\rho_a) \\
	& \qquad \mathcal{N}(b_t- b_{t-1}\mid 0, \rho_b I)
	\mathcal{N}(\theta_{t-1}\mid \hat\theta_{t-1\mid t-1},P_{t-1\mid t-1}) \\
	& \qquad
	\mathcal{N}(a_t\mid \hat a_{t-1\mid t-1},s_{t-1\mid t-1})
	\mathcal{N}(b_t\mid \hat b_{t-1\mid t-1},\Sigma_{t-1\mid t-1}) \\
	& \qquad\qquad d\theta_{t-1} da_{t-1} db_{t-1} \\
	& \approx 
	\mathcal{N}(\theta_t\mid K\hat\theta_{t-1\mid t-1},KP_{t-1\mid t-1}K^\top + f(b_t)) \\
	& \qquad\qquad \mathcal{N}(a_t\mid \hat a_{t-1\mid t-1}, s_{t-1\mid t-1}+\rho_a) \\
	& \qquad\qquad
	\mathcal{N}(b_t\mid \hat b_{t-1\mid t-1}, \Sigma_{t-1\mid t-1}+\rho_b I) \,.
\end{align*}

Treating the approximation at time $t-1$ as a prior at time $t$ we obtain the following posterior distribution:
\begin{align}
	\nonumber
	& p(\theta_t,a_t,b_t\mid \mathcal{F}_t) = 
	\mathcal{N}(y_t\mid \theta_t^\top x_t,\exp(a_t)) \\
	\nonumber
	& \qquad
	\mathcal{N}(\theta_t\mid K \hat\theta_{t-1\mid t-1},K P_{t-1\mid t-1} K^\top + f(b_t)) \\
	\nonumber
	& \qquad
	\mathcal{N}(a_t\mid \hat a_{t-1\mid t-1},s_{t-1\mid t-1}+\rho_a) \\
	\label{eq:posterior} & \qquad
	\mathcal{N}(b_t\mid \hat b_{t-1\mid t-1},\Sigma_{t-1\mid t-1}+\rho_b I)
	\frac{p(x_t,\mathcal{F}_{t-1})}{p(\mathcal{F}_t)} \,.
\end{align}
This last equation defines the posterior that we plug in \eqref{eq:kl} to obtain the optimization problem that we would like to solve recursively.

The term $\mathcal{N}(\theta_t\mid K \hat\theta_{t-1\mid t-1},K P_{t-1\mid t-1} K^\top + Q_t)$ on the second line, which would appear with any model for $Q_t$, makes a conjugate prior for $Q_t$ impractical in a VB method to estimate the posterior distribution of the state and the variances.
The approach proposed by \cite{huang2020slide} consists in applying a few iterations of Kalman smoothing with the previous estimates of the variances $\hat\sigma_{t-1}^2$ and $\hat Q_{t-1}$. Then the authors estimate the posterior distribution of $\sigma_t^2,Q_t$ given $\mathcal{F}_t$ and the distribution of $\theta_{t-k}$ estimated by Kalman smoothing given $\mathcal{F}_t,\hat\sigma_{t-1}^2,\hat Q_{t-1}$. In that way they get rid of the crossed factor involving $\theta_t$ and $Q_t$, and they obtain exact estimation of the posterior distribution of the variances.
Our approach does the opposite on that part. We build on that crossed factor to avoid Kalman smoothing, at the cost of the need of approximations in the posterior estimation.

\subsection{KL Derivation and Optimum in $\hat\theta_{t\mid t},P_{t\mid t}$}\label{sec:kl}
We first present a detailed expression of the KL divergence defined in \eqref{eq:kl} in the following Lemma.
\begin{lemma}\label{lemma:kl_expression}
There exists a constant $c$ independent of $\hat\theta_{t\mid t},P_{t\mid t},\hat a_{t\mid t},s_{t\mid t},\hat b_{t\mid t},\Sigma_{t\mid t}$ such that
\begin{align*}
	& KL\Big(\mathcal{N}(\hat\theta_{t\mid t},P_{t\mid t})
	\mathcal{N}(\hat a_{t\mid t},s_{t\mid t})
	\mathcal{N}(\hat b_{t\mid t},\Sigma_{t\mid t})\ ||\ 
	p(\cdot\mid \mathcal{F}_t)\Big) \\
	&\quad  = -\frac12 \log\det P_{t\mid t} -\frac12 \log(s_{t\mid t})  \\
	& \qquad + \frac12 ((y_t-\hat\theta_{t\mid t}^\top x_t)^2 + x_t^\top P_{t\mid t} x_t) \exp(-\hat a_{t\mid t}+\frac12s_{t\mid t}) \\
	& \qquad -\frac12 \log\det\Sigma_{t\mid t} + \frac12 \mathbb{E}_{b_t\sim\mathcal{N}(\hat b_{t\mid t},\Sigma_{t\mid t})}[\psi_t(b_t)] \\
	& \qquad + \frac{1}{2(s_{t-1\mid t-1}+\rho_a)}(s_{t\mid t}+(\hat a_{t\mid t}-\hat a_{t-1\mid t-1})^2) + \frac12 \hat a_{t\mid t} \\
	& \qquad + \frac12 \Tr\Big(\big(\Sigma_{t\mid t} + (\hat b_{t\mid t}-\hat b_{t-1\mid t-1})(\hat b_{t\mid t}-\hat b_{t-1\mid t-1})^\top\big) \\
	& \qquad\qquad\qquad\qquad (\Sigma_{t-1\mid t-1} + \rho_b I)^{-1} \Big)+ c \,,
\end{align*}
where
\begin{align*}
    & \psi_t(b_t) =\log\det(KP_{t-1\mid t-1}K^\top + f(b_t)) \\
    & \qquad + \Tr\Big((P_{t\mid t} + (\hat\theta_{t\mid t}-K\hat\theta_{t-1\mid t-1})(\hat\theta_{t\mid t}-K\hat\theta_{t-1\mid t-1})^\top) \\
    & \qquad\qquad\qquad (KP_{t-1\mid t-1}K^\top + f(b_t))^{-1}\Big) \,.
\end{align*}
\end{lemma}

We easily obtain a closed-form solution to minimize the KL divergence with respect to $\hat\theta_{t\mid t},P_{t\mid t}$.
\begin{theorem}\label{th:optimum_theta}
Given $\hat a_{t\mid t},s_{t\mid t}, \hat b_{t\mid t},\Sigma_{t\mid t}$, the values of $\hat\theta_{t\mid t},P_{t\mid t}$ minimizing the KL divergence are given by
\begin{align}
	\label{eq:defA}
	A_t & = \mathbb{E}_{b_t\sim\mathcal{N}(\hat b_{t\mid t},\Sigma_{t\mid t})} [(KP_{t-1\mid t-1}K^\top + f(b_t))^{-1}] \,,\\
	\label{eq:updateP}
	P_{t\mid t} & = A_t^{-1} - \frac{A_t^{-1}x_tx_t^\top A_t^{-1}}{x_t^\top A_t^{-1} x_t + \exp(\hat a_{t\mid t} - \frac12 s_{t\mid t})} \,, \\
	\label{eq:updatetheta}
	\hat\theta_{t\mid t} & = K \hat\theta_{t-1\mid t-1} + \frac{P_{t\mid t} x_t}{e^{\hat a_{t\mid t} - s_{t\mid t}/2}} (y_t - x_t^\top K \hat\theta_{t-1\mid t-1}) \,.
\end{align}
\end{theorem}
Note that the updates defined above are the ones of the Kalman filter with known variances $\sigma_t^2$ and $Q_t$, where we have replaced $\sigma_t^2$ with $\exp(\hat a_{t\mid t} - \frac12 s_{t\mid t})$ which is $\mathbb{E}_{a_t\sim\mathcal{N}(\hat a_{t\mid t},s_{t\mid t})}[\exp(a_t)^{-1}]^{-1}$ and $KP_{t-1\mid t-1}K^\top+Q_t$ with $\mathbb{E}_{b_t\sim\mathcal{N}(\hat b_{t\mid t},\Sigma_{t\mid t})} [(KP_{t-1\mid t-1}K^\top + f(b_t))^{-1}]^{-1}$. If $s_{t\mid t}=0,\Sigma_{t\mid t}=0$ then we know the variances and we obtain the Kalman filter with $\sigma_t^2=\exp(\hat a_{t\mid t})$ and $Q_t=f(\hat b_{t\mid t})$. Otherwise if $\Sigma_{t\mid t}\neq 0$, the result states that the update of the Kalman filter with unbiased estimated variances in place of the unknown variances is suboptimal in the sense of the Kullback-Leibler divergence. It implies also that we don't expect to obtain unbiased estimates of the variances. The same conclusion would follow if one adapted the classical VB approach of \cite{tzikas2008variational} to our framework.

It is important to remark that as long as $\rho_b>0$ we do not have the convergence of $\Sigma_{t\mid t}$ to $0$. Therefore we do not recover the standard Kalman filter asymptotically. On the contrary, existing adaptive Kalman filters use the standard Kalman recursive updates with estimates of the variances \cite{sarkka2009recursive, agamennoni2012approximate, sarkka2013non, huang2017novel, huang2020slide}. Therefore, in a well-specified setting where the state-space model is the underlying generating process, our method should be outperformed by adaptive Kalman filters whose variance estimates are consistent. We believe this drawback is a reasonable price to pay to get robustness to misspecification.

Furthermore note that \eqref{eq:updatetheta} may be interpreted as a gradient step on the quadratic loss, where instead of a gradient step size we have the preconditioning matrix $P_{t\mid t}/\exp(\hat a_{t\mid t} - \frac12 s_{t\mid t})$. Therefore the algorithm derived in this article may be seen as a way to parameterize a second order stochastic gradient algorithm.

\subsection{Choice of $f$}\label{sec:choice_f}
The natural transform for the latent variables $a_t$ and $b_t$ is the exponential, see \cite{tyagi2008recursive} for a filter on latent variables lying in a Riemannian manifold. We use the exponential to represent $\sigma_t^2$. However setting $f(b_t)=\exp(b_t)I$ for a unidimensional $b_t$ contradicts a careful property that we define as follows using the gradient interpretation of Section \ref{sec:kl}. We claim that the algorithm should be more careful with uncertainty ($\Sigma_{t\mid t}\succ 0$) than without ($\Sigma_{t\mid t}=0$).
By more careful we mean smaller gradient steps, that is formally $A_t^{-1} \preccurlyeq KP_{t-1\mid t-1}K^\top + f(\hat b_{t\mid t})$.
By Jensen's inequality we have
\begin{align*}
    A_t\succcurlyeq \Big(KP_{t-1\mid t-1}K^\top + \mathbb{E}_{b_t\sim \mathcal{N}(\hat b_{t\mid t},\Sigma_{t\mid t})}[f(b_t)]\Big)^{-1} \,.
\end{align*}
Therefore a sufficient (but not necessary) condition providing the careful property is $f$ concave, again thanks to Jensen, and that is the contrary of the exponential. Unfortunately we cannot have both $f$ concave and $f\succcurlyeq 0$. We propose to use a function which is zero on negative numbers and concave elsewhere:
\begin{align*}
    \phi(b) = 
    \begin{cases}
    0 & if\ b<0\,, \\
    \log(1+b) & if\ b\ge 0\,.
    \end{cases}
\end{align*}
Then we consider two settings for $f$: First a scalar setting where $f(b_t)=\phi(b_t)I$ for a unidimensional $b_t$. Second, a diagonal setting where $b_t\in\mathbb{R}^d$ and $f(b_t)=D_{\phi(b_t)}$ is a diagonal matrix whose diagonal coefficients are defined by the $\phi$ transform applied to each coefficient of $b_t$.

\section{Approximate Variational Bayes}\label{sec:approx_vb}
Theorem \ref{th:optimum_theta} realizes the exact optimum of the KL divergence with respect to $\hat\theta_{t\mid t},P_{t\mid t}$. To obtain closed-form solutions of the minimum with respect to the other parameters we need additional approximations. In this section, we use the first two moments of gaussian distributions in second-order upper-bounds. That yields closed-form approximations to the VB recursive step with respect to $\hat a_{t\mid t},s_{t\mid t}$ and $\hat b_{t\mid t},\Sigma_{t\mid t}$. Minimizing the upper-bounds does not necessarily lead to minimizing the KL divergence, but it yields the guarantee of decreasing the instantaneous KL divergence at each step.

\subsection{Optimum in $\hat a_{t\mid t},s_{t\mid t}$}\label{sec:opt_a}
We first present recursive updates for $\hat a_{t\mid t},s_{t\mid t}$.

\subsubsection{Optimum in $s_{t\mid t}$}
We are looking for $s_{t\mid t}\ge 0$ minimizing the KL divergence. As the conditional variance of $a_t$ given $\mathcal{F}_{t-1}$ is $s_{t-1\mid t-1}+\rho_a$, we look for $s_{t\mid t}$ in the interval $[0,s_{t-1\mid t-1}+\rho_a]$. In this interval we simply use a linear upper-bound for the exponential:
\begin{proposition}\label{prop:optimum_s}
    For any $s_{t\mid t}\in[0,s_{t-1\mid t-1}+\rho_a]$ we have 
    \begin{align*}
	    & KL\Big(\mathcal{N}(\hat\theta_{t\mid t},P_{t\mid t})
	    \mathcal{N}(\hat a_{t\mid t},s_{t\mid t})
	    \mathcal{N}(\hat b_{t\mid t},\Sigma_{t\mid t})\ ||\ p(\cdot\mid \mathcal{F}_t)\Big) \\
	    & \quad \le \frac14 ((y_t-\hat\theta_{t\mid t}^\top x_t)^2 + x_t^\top P_{t\mid t} x_t) e^{-\hat a_{t\mid t}} s_{t\mid t} \\
	    & \qquad +\frac12 (s_{t-1\mid t-1}+\rho_a)^{-1} s_{t\mid t} -\frac12 \log(s_{t\mid t}) + c_s \,,
    \end{align*}
    where $c_s$ is a constant independent of $s_{t\mid t}$. Furthermore, the upper-bound is minimized by:
    \begin{align}
	    \nonumber
	    & s_{t\mid t} = \Big( (s_{t-1\mid t-1}+\rho_a)^{-1} \\
	    \label{eq:updates}
	    & \qquad\qquad + \frac12 ((y_t-\hat\theta_{t\mid t}^\top x_t)^2 + x_t^\top P_{t\mid t} x_t) e^{-\hat a_{t\mid t}} \Big)^{-1} \,.
     \end{align}
\end{proposition}

\subsubsection{Optimum in $\hat a_{t\mid t}$}
To upper-bound the exponential with a polynomial form also in $\hat a_{t\mid t}$ we need to bound $\hat a_{t\mid t}$, and we consider the segment $[\hat a_{t-1\mid t-1}- M_a,\hat a_{t-1\mid t-1}+ M_a]$ (we set arbitrarily $M_a=3s_{t-1\mid t-1}$).
\begin{proposition}\label{prop:optimum_a}
    For any $\hat a_{t\mid t}\in[\hat a_{t-1\mid t-1}- M_a,\hat a_{t-1\mid t-1}+ M_a]$ we have 
    \begin{align*}
	    & KL\Big(\mathcal{N}(\hat\theta_{t\mid t},P_{t\mid t})
	    \mathcal{N}(\hat a_{t\mid t},s_{t\mid t})
	    \mathcal{N}(\hat b_{t\mid t},\Sigma_{t\mid t})\ ||\ p(\cdot\mid \mathcal{F}_t)\Big) \\
	    & \quad \le \frac12 ((y_t-\hat\theta_{t\mid t}^\top x_t)^2 + x_t^\top P_{t\mid t} x_t) e^{-\hat a_{t-1\mid t-1}+ s_{t\mid t}/2}  \\
	    & \qquad\qquad \Big(- (\hat a_{t\mid t} - \hat a_{t-1\mid t-1}) + \frac{e^{M_a}}{2}(\hat a_{t\mid t} - \hat a_{t-1\mid t-1})^2\Big) \\
	    & \qquad +\frac12 (s_{t-1\mid t-1}+\rho_a)^{-1} (\hat a_{t\mid t}-\hat a_{t-1\mid t-1})^2 +\frac12 \hat a_{t\mid t} + c_a \,,
    \end{align*}
    where $c_a$ is a constant independent of $\hat a_{t\mid t}$. Furthermore the upper-bound is minimized by:
    \begin{align}
	    \nonumber
	    & \hat a = \hat a_{t-1\mid t-1} + \frac12\Big(\frac{1}{s_{t-1\mid t-1}+\rho_a} \\
	    \nonumber
	    & \quad + \frac12 ((y_t-\hat\theta_{t\mid t}^\top x_t)^2 + x_t^\top P_{t\mid t} x_t) e^{-\hat a_{t-1\mid t-1}  + s_{t\mid t}/2 + M_a}  \Big)^{-1} \\
	    \nonumber
	    & \qquad \Big(((y_t-\hat\theta_{t\mid t}^\top x_t)^2 + x_t^\top P_{t\mid t} x_t) e^{-\hat a_{t-1\mid t-1} + s_{t\mid t}/2} - 1\Big) \,, \\
	    \label{eq:updatea}
	    & \hat a_{t\mid t} = \max(\min(\hat a, \hat a_{t-1\mid t-1} + M_a), \hat a_{t-1\mid t-1}- M_a) \,.
    \end{align}
\end{proposition}
We note that $((y_t-\hat\theta_{t\mid t}^\top x_t)^2 + x_t^\top P_{t\mid t} x_t) e^{-\hat a_{t-1\mid t-1} + s_{t\mid t}/2} - 1$ is the gradient with respect to $\hat a$ of
\begin{align*}
    & \mathbb{E}_{(\theta_t,a_t)\sim \mathcal{N}(\hat\theta_{t\mid t},P_{t\mid t})\times \mathcal{N}(\hat a,s_{t\mid t})}[\log\mathcal{N}(y_t\mid \theta_t^\top x_t, \exp(a_t))] \\
    & \quad = -\frac12 \hat a - \frac12 ((y_t-\hat\theta_{t\mid t}^\top x_t)^2 + x_t^\top P_{t\mid t} x_t) e^{-\hat a + s_{t\mid t}/2} \,,
\end{align*}
therefore \eqref{eq:updatea} may be seen as a projected gradient step on an expected log-likelihood.

\subsection{Optimum in $\hat b_{t\mid t}, \Sigma_{t\mid t}$}\label{sec:opt_b}
The minimum of the Kullback-Leibler is also intractable in $\hat b_{t\mid t}, \Sigma_{t\mid t}$ due to the absence of analytical form for the expected value of $\psi_t$. In the following we focus on the specific settings that are introduced in Section \ref{sec:choice_f}, namely the scalar setting $f(b_t)=\phi(b_t)I$ and the diagonal setting $f(b_t)=D_{\phi(b_t)}$. For these two possible choices of $f$ we have the following second-order upper-bound for $\psi_t$:
\begin{proposition}\label{prop:upper}
In the scalar and diagonal settings defined in Section \ref{sec:choice_f}, for any $t$ such that $f(\hat b_{t-1\mid t-1})\succ 0$, the following holds for any $b_t$ in a neighbourhood of $\hat b_{t-1\mid t-1}$:
\begin{align*}
	\psi_t(b_t) & \le \psi_t(\hat b_{t-1\mid t-1}) + \frac{\partial \psi_t}{\partial b_t}\Big|_{\substack{\hat b_{t-1\mid t-1}}}^\top (b_t-\hat b_{t-1\mid t-1}) \\
	& \qquad + \frac12 (b_t - \hat b_{t-1\mid t -1})^\top H_t (b_t - \hat b_{t-1\mid t -1}) \,,
\end{align*}
where $B_t = P_{t\mid t} + (\hat\theta_{t\mid t}-K\hat\theta_{t-1\mid t-1})(\hat\theta_{t\mid t}-K\hat\theta_{t-1\mid t-1})^\top$, $C_t=KP_{t-1\mid t-1}K^\top+f(\hat b_{t-1\mid t-1})$, and then
\begin{align*}
	& \frac{\partial \psi_t}{\partial b_t}\Big|_{\substack{\hat b_{t-1\mid t-1}}} = \Tr(C_t^{-1}(I - B_tC_t^{-1}))\phi'(\hat b_{t-1\mid t-1}) \,, \\
    & H_t = - \Tr(C_t^{-1}B_tC_t^{-1})\phi''(\hat b_{t-1\mid t-1}) \hspace{1.8cm} \\
    & \qquad + 2 \Tr(C_t^{-2}B_tC_t^{-1}) \phi'(\hat b_{t-1\mid t-1})^2 \,,
\end{align*}
in the scalar setting, and
\begin{align*}
    & \frac{\partial \psi_t}{\partial b_t}\Big|_{\substack{\hat b_{t-1\mid t-1}}} =
	\Delta_{C_t^{-1}(I - B_tC_t^{-1})} \odot \phi'(\hat b_{t-1\mid t-1}) \,,\\
	& H_t = - \Big(C_t^{-1} B_tC_t^{-1} D_{\phi''(\hat b_{t-1\mid t-1})}\Big) \odot I \hspace{1cm} \\
    & \qquad + 2 C_t^{-1}B_tC_t^{-1} \odot C_t^{-1} \odot \phi'(\hat b_{t-1\mid t-1})\phi'(\hat b_{t-1\mid t-1})^\top \,,
\end{align*}
in the diagonal setting, with $\odot$ the Hadamard (pointwise) product.
\end{proposition}
The upper-bound of the Kullback-Leibler divergence obtained thanks to the proposition above admits a closed-form minimum:
\begin{proposition}\label{prop:optimum_b}
In the scalar and diagonal settings, for any $t$ such that $f(\hat b_{t-1\mid t-1})\succ 0$ and any $\hat b_{t\mid t},\Sigma_{t\mid t}$,
\begin{align*}
	& KL\Big(\mathcal{N}(\hat\theta_{t\mid t},P_{t\mid t})
	\mathcal{N}(\hat a_{t\mid t},s_{t\mid t})
	\mathcal{N}(\hat b_{t\mid t},\Sigma_{t\mid t})\ ||\ 
	p(\cdot\mid \mathcal{F}_t)\Big) \\
	&\quad \le -\frac12 \log\det\Sigma_{t\mid t} + \frac12 \frac{\partial \psi_t}{\partial b_t}\Big|_{\substack{\hat b_{t-1\mid t-1}}}^\top (\hat b_{t\mid t}-\hat b_{t-1\mid t-1}) \\
	& \qquad + \frac12 \Tr\Big( (\Sigma_{t\mid t} + (\hat b_{t\mid t}-\hat b_{t-1\mid t-1})(\hat b_{t\mid t}-\hat b_{t-1\mid t-1})^\top ) \\
	& \qquad\qquad\qquad \Big((\Sigma_{t-1\mid t-1} + \rho_b I)^{-1} + \frac12 H_t \Big) \Big) + c_b \,,
\end{align*}
where $H_t$ is defined in Proposition \ref{prop:upper} and $c_b$ is a constant independent of $\hat b_{t\mid t},\Sigma_{t\mid t}$. The minimum of the upper-bound detailed above is obtained with:
\begin{align}
	\label{eq:updateSigma}
	& \Sigma_{t\mid t} = \Big((\Sigma_{t-1\mid t-1} + \rho_b I)^{-1} +\frac12 H_t\Big)^{-1} \,, \\
	\label{eq:updateb}
	& \hat b_{t\mid t} = \hat b_{t-1\mid t-1} - \frac12 \Sigma_{t\mid t} \frac{\partial \psi_t}{\partial b_t}\Big|_{\substack{\hat b_{t-1\mid t-1}}} \,.
\end{align}
\end{proposition}
Similarly as \eqref{eq:updatea} we can interpret \eqref{eq:updateb} as a gradient step on $\psi_t$ and we can remark that $\psi_t(\hat b)$ is the following expected log-likelihood:
\begin{align*}
    \mathbb{E}_{\theta_t\sim\mathcal{N}(\hat\theta_{t\mid t},P_{t\mid t})}[\log\mathcal{N}(\theta_t\mid K \hat\theta_{t-1\mid t-1}, KP_{t-1\mid t-1}K^\top +f(\hat b))] \,.
\end{align*}
Thus, except the exact recursive steps on $\hat\theta_{t\mid t},P_{t\mid t}$ which are extensions of the Kalman filter steps, our recursive steps resemble Stochastic Gradient Variational Bayes (SGVB) algorithm steps as described in \cite{knowles2015stochastic}. This novel class of algorithms is very popular for tuning complex deep learning networks, see for instance \cite{kingma2014stochastic,tjandra2015stochastic}. There, the expectation of the log-likelihood is approximated by using Monte-Carlo simulation and only the first order of the gradient is used.

\section{Viking}\label{sec:algorithm}
We now introduce the algorithm following from the recursive updates described in the previous sections.

\subsection{Definition of the Algorithm}
Theorem \ref{th:optimum_theta} yields exact recursive updates for $\hat\theta_{t\mid t},P_{t\mid t}$ but $A_t^{-1}$ does not admit an explicit form. We propose to run Monte-Carlo estimation of $A_t$ with very small samples ($n_{\rm mc}=10$ draws by default). As the KL optimization is a coupled problem we solve it in a classical iterative fashion, that is, we repeat $N$ times the updates alternately ($N=2$ is a good default value).
We summarize the procedure in Algorithm \ref{alg:viking}. We name it Viking (\textbf{V}ariational Bayes\textbf{i}an Variance Trac\textbf{king}).

\begin{algorithm}[!t]
{\caption{Viking at time step $t$}
\label{alg:viking}}
{
{\bf Time-invariant parameters:} $\rho_a,\rho_b,n_{\rm mc},f$.\\
{\bf Default:} $\rho_a=e^{-9},\rho_b=e^{-6},n_{\rm mc}=10,f(\cdot)=D_{\phi(\cdot)}$.\\
{\bf Inputs}: $\hat\theta_{t-1\mid t-1}$, $P_{t-1\mid t-1}$, $\hat a_{t-1\mid t-1}$, $s_{t-1\mid t-1}$, $\hat b_{t-1\mid t-1}$, $\Sigma_{t-1\mid t-1}$, $x_t$, $y_t$. \\
{\bf Initialize}: \\
Set $\hat a_{t\mid t}^{(0)}=\hat a_{t-1\mid t-1}$, $s_{t\mid t}^{(0)}=s_{t-1\mid t-1}+\rho_a$. \\
Set $\hat b_{t\mid t}^{(0)}=\hat b_{t-1\mid t-1}$, $\Sigma_{t\mid t}^{(0)}=\Sigma_{t-1\mid t-1}+\rho_b$.\\
{\bf Iterate: for} $i=1,\hdots,N$:
\begin{itemize}
\item
1. Set $A_t$ then compute $A_t^{-1}$ using \eqref{eq:defA} with \\ \hspace{0.4cm}Monte-Carlo from $n_{\rm mc}$ samples of $\mathcal{N}(\hat b_{t\mid t}^{(i-1)},\Sigma_{t\mid t}^{(i-1)})$.\\
2. Set $P_{t\mid t}^{(i)},\hat\theta_{t\mid t}^{(i)}$ using \eqref{eq:updateP} and \eqref{eq:updatetheta}, with $A_t^{-1}$ from step 1 \\
\hspace{0.4cm}and $\hat a_{t\mid t}^{(i-1)},s_{t\mid t}^{(i-1)}$.
\item
{\bf If we learn $\sigma_t^2$:}\\
3. Set $s_{t\mid t}^{(i)}$ using \eqref{eq:updates} with $\hat\theta_t^{(i)},P_t^{(i)},\hat a_{t\mid t}^{(i-1)}$.\\
4. Set $\hat a_{t\mid t}^{(i)}$ using \eqref{eq:updatea} with $\hat\theta_t^{(i)},P_t^{(i)},s_{t\mid t}^{(i)}$.
\item
{\bf If we learn $Q_t$:}\\
5. Set $\Sigma_{t\mid t}^{(i)},\hat b_{t\mid t}^{(i)}$ using \eqref{eq:updateSigma} and \eqref{eq:updateb}.\\
\hspace{0.4cm}Apply threshold $\hat b_{t\mid t}^{(i)}=\max(\hat b_{t\mid t}^{(i)},0)$.
\end{itemize}
{\bf Outputs:} $\hat\theta_{t\mid t}=\hat\theta_{t\mid t}^{(N)},P_{t\mid t}=P_{t\mid t}^{(N)},\hat a_{t\mid t}=\hat a_{t\mid t}^{(N)},s_{t\mid t}=s_{t\mid t}^{(N)},\hat b_{t\mid t}=\hat b_{t\mid t}^{(N)},\Sigma_{t\mid t}=\Sigma_{t\mid t}^{(N)}$.
}
\end{algorithm}

\subsection{Complexity}\label{sec:complex}
We decompose the number of operations of Viking in Table \ref{tab:complexity}. Although matrix multiplication and inversion have the same asymptotic complexity, in practice inversion is more costly.
\begin{table}[t!]
    \begin{center}
    \caption{Complexity of Algorithm \ref{alg:viking}.}
    \label{tab:complexity}
    \begin{tabular}{c c}
        \hline
        Steps & Operations \\
        \hline
        1 & $n_{\rm mc} S + (n_{\rm mc}+1)I(d) + \mathcal{O}(M(d))$ \\
        2 & $\mathcal{O}(d^2)$ \\
        3 and 4 & $\mathcal{O}(d^2)$ \\
        5 & $3I(d) + \mathcal{O}(M(d))$ \\
        \hline
        Whole & $N\big(n_{\rm mc} S+(n_{\rm mc}+4)I(d)+\mathcal{O}(M(d))\big)$ \\
        \hline
    \end{tabular}
    \end{center}
    $S$ denotes the complexity of gaussian draw, $M(d)$ and $I(d)$ denote the complexity of matrix multiplication and inversion.
\end{table}

We suggest the default $N=2$ and $n_{\rm mc}=10$, therefore the complexity of Viking is essentially driven by the complexity of matrix inversion. Consequently it is proportional to the one of methods relying on Kalman smoothing as in \cite{huang2020slide}.

\section{Experiments}\label{sec:experiments}
We run several experiments, and we argue that our method behaves well for misspecified data. We begin with well-specified data generated under a state-space model with smoothly varying variances. Then we focus on misspecified data.

\subsection{Well-Specified Data with Unknown $\sigma_t^2$ and Known $Q_t$}\label{sec:exp_sarkka}
We reproduce the experiment presented in \cite{sarkka2009recursive} on the stochastic resonator model:
\begin{align*}
	& \theta_{t+1} - \begin{pmatrix}
1 & 0 & 0 \\
0 & \cos(\omega\Delta t) & \frac{\sin(\omega\Delta t)}{\omega} \\
0 & - \omega \sin(\omega\Delta t) & \cos(\omega\Delta t)
\end{pmatrix}
	\theta_t \sim\mathcal{N}(0,Q) \,, \\
	& y_t - (\theta_{t,1} + \theta_{t,2}) \sim\mathcal{N}(0,\sigma_t^2) \,,
\end{align*}
where we set $\omega=0.05$ and $\Delta t=0.1$ and the known covariance of the process noise is $Q=D_{(0.01,0,0.0001)}$. We display the variance trajectory for one simulation in Figure \ref{fig:sarkkatraj}.
\begin{figure}
	\centering
	\includegraphics[width=7cm]{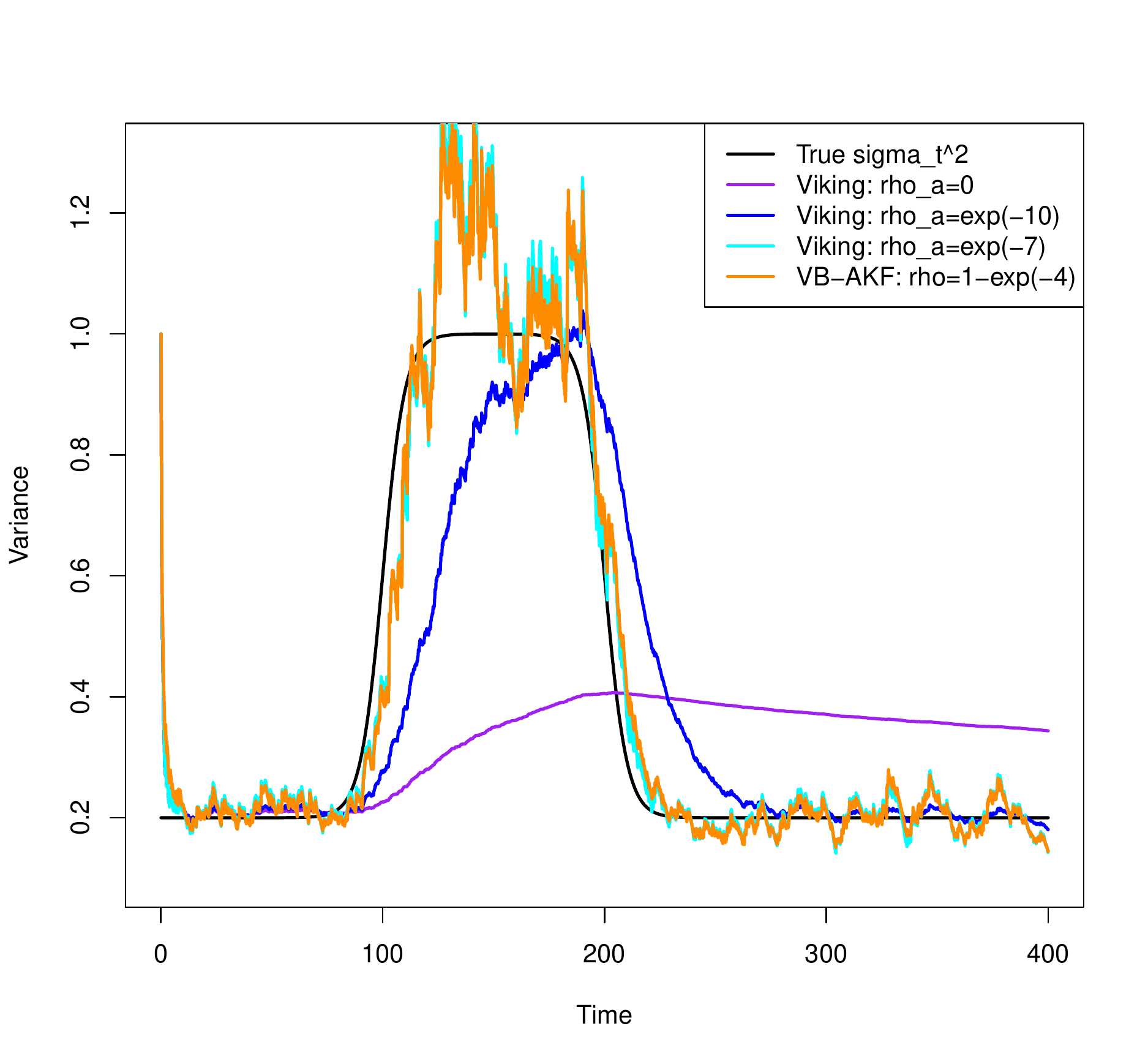}
	\caption{Trajectory of the observation variance $\sigma_t^2$ estimated by our algorithm and compared to the estimate provided in \cite{sarkka2009recursive}.}
	\label{fig:sarkkatraj}
\end{figure}
Running the experiment $100$ times we observe that both methods almost coincide in terms of root-mean-square error: 0.6859 for Viking and 0.6858 for VB-AKF \cite{sarkka2009recursive}. In this comparison we take the best value of $\rho_a$ for Viking as well as the best $\rho$ for the VB-AKF in the list $e^{-i},1\le i\le 10$.

\subsection{Well-Specified Data with Unknown $\sigma_t^2$ and $Q_t$}\label{sec:exp_synws}
We run a second simulation inspired by \cite{huang2020slide} in a well-specified setting. We generate $x_t\in[0,1]^5$ using two possible alternatives:
\begin{enumerate}
\item
{\bf Uniform i.i.d. design:} $(x_t)$ is independent identically distributed. For each $t$, $x_t$ is composed of $4$ independent coefficients generated with uniform distributions on $[0,1]$ and one deterministic $1$ coefficient.
\item
{\bf Uniform non-i.i.d. design:} $(x_t)$ has the same distribution but is not i.i.d., a sample is displayed in Figure \ref{fig:x_uninoniid}. Precisely $x_1$ is generated as before. Then for $j\in\{1,2,3,4\}$ and $t\ge 2$, we consider $z_{t,j}=x_{t-1,j}+\varepsilon_{t,j}$ where $\varepsilon_{t,j}\sim\mathcal{N}(0,10^{-3})$ and we generate
\begin{align*}
	x_{t,j} = \begin{cases}
	z_{t,j} \text{ if } 0\le z_{t,j} \le 1\,, \\
	\lceil z_{t,j} \rceil - z_{t,j} \text{ otherwise}.
	\end{cases}
\end{align*}
\begin{figure}
	\centering
	\includegraphics[width=7cm]{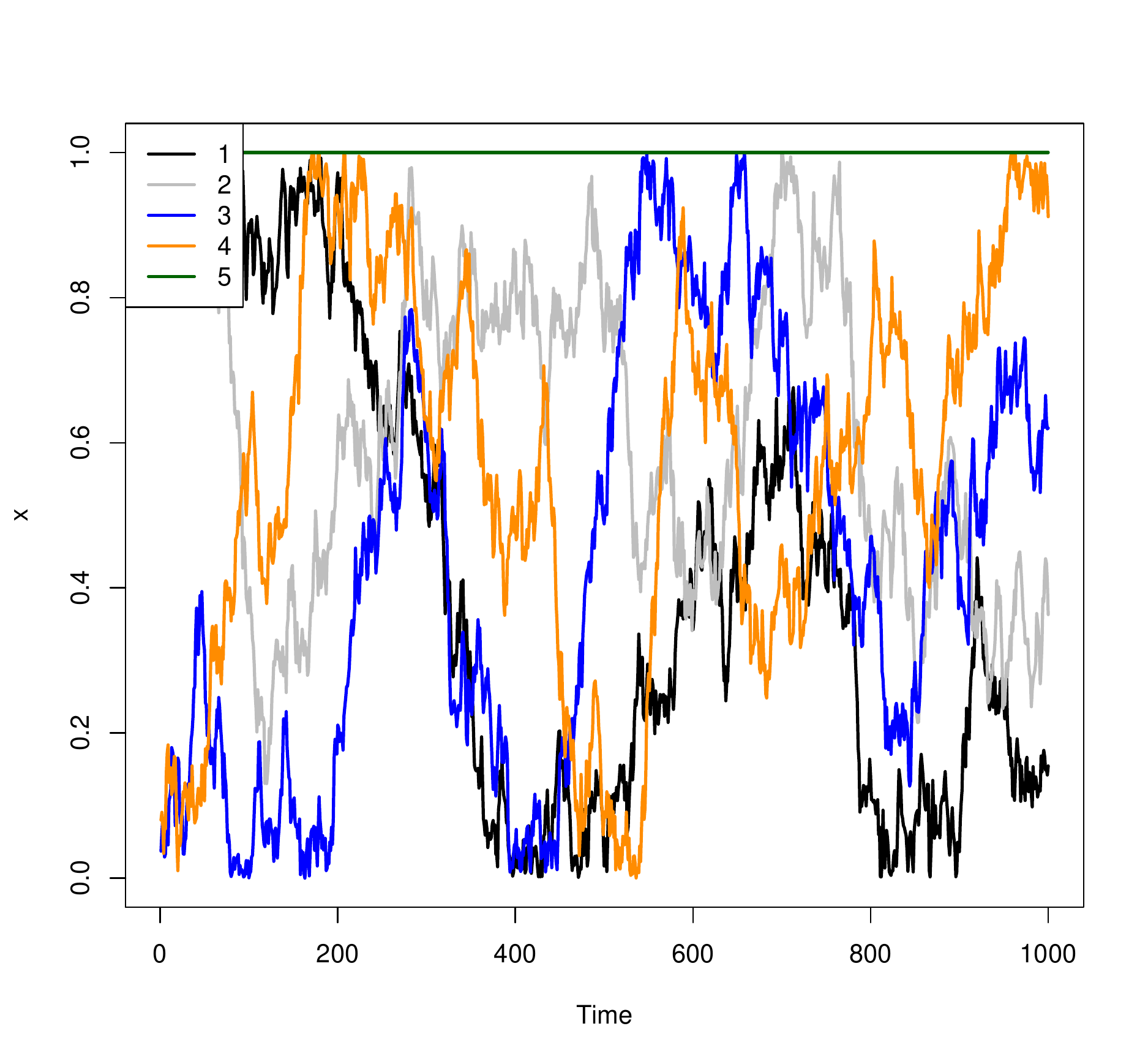}
	\caption{Example of trajectory of the 5 components of the vector $x_t$ considered in the setting {\it uniform non-iid}.}
	\label{fig:x_uninoniid}
\end{figure}
\end{enumerate}
Then we generate $y_t$ by the following state-space model:
\begin{align*}
	& \theta_0\sim\mathcal{N}(0,I) \,, \\
    & \theta_t - \theta_{t-1} \sim \mathcal{N}(0,Q_t) \,, \\
    & y_t - \theta_t^\top x_t \sim \mathcal{N}(0,\sigma_t^2) \,,
\end{align*}
where
\begin{align*}
	& \sigma_t^2=1+0.1\cos{\frac{4\pi t}{n}} \,, \\
	& Q_t=\Big(0.25+0.2\cos{\frac{4\pi t}{n}}\Big) D_{(0,0,1,1,1)} \,.
\end{align*}
The simulation time is $n=10^3$. In Figure \ref{fig:mse_both} we compare Viking to the slide window variational adaptive Kalman filter (SWVAKF) introduced in \cite{huang2020slide}, which we tune in several ways. First we increase the window length from 5 to 20, resulting in a significant improvement at the cost of more computations. Second we tune the forgetting factor, and to play fair with Viking we define different forgetting factors for the estimation of $\sigma_t^2$ and the estimation of $Q_t$. We select the best {\it a posteriori}, and we do the same for Viking. Third, we enforce diagonal and scalar variants of the SWVAKF: the diagonal variant is defined by replacing by 0 each non-diagonal coefficient after each update, and on top of that in the scalar variant we replace each diagonal coefficient by the averaged diagonal.

\subsection{Misspecified Data with Unknown $\sigma_t^2$ and $Q_t$}\label{sec:exp_synms}
To experiment misspecification we consider a state-space model with two states evolving independently with identical processes, and the observation is generated using one of them uniformly at random. That is summarized by the following set of equations:
\begin{align*}
	& \theta_0^{(i)}\sim\mathcal{N}(0,I) \,, & i\in\{0,1\}\,, \\
	& \theta_t^{(i)}-0.9\,\theta_t^{(i)} \sim \mathcal{N}(0,Q_t)\,, & i\in\{0,1\}\,, \\
	& i_t\sim\mathcal{B}(1/2) \,, & \\
	& y_t - \theta_t^{(i_t)\top} x_t \sim\mathcal{N}(0,\sigma_t^2) \,, &
\end{align*}
where we assume all gaussian noises to be independent of each other and of $(i_t)$. We consider the same settings for $x_t$ as well as the same variances $\sigma_t^2,Q_t$ defined in Section \ref{sec:exp_synws}.

The contraction (here by a coefficient 0.9) is necessary to have the convergence of the distribution of $y_t$ as well as of the conditional distribution of $y_t$ given the filtration $\mathcal{F}_{t-1}$. In the tracking mode (no contraction) the variance of the conditional distribution would diverge to $\infty$, and therefore the error of any forecasting strategy would also diverge to $\infty$.

We refer to Figure \ref{fig:mse_both} for the evaluation in mean squared error. We observe that Viking in the diagonal setting behaves poorly compared to the SWVAKF for well-specified data with i.i.d. design but better in the other 3 experiments. As mentioned in Section \ref{sec:kl} we believe it is natural that a consistent adaptive Kalman filter should be closer to the true Kalman filter than our algorithm which cannot be written using Kalman recursion. However the careful property (see the design of $f$ in Section \ref{sec:choice_f}) allows us to outperform existing methods for misspecified data. This interpretation of the observation generation may to a minor extent be transposed to the design generation. Indeed, in our non-i.i.d. design a shift in the data should be harder to attribute to one coefficient of the state, and therefore it should be harder to learn the variances, that is why the difference between the two Kalman filters with constant variances is smaller. Thus the model should not be trusted too much.

\begin{figure*}
	\centering
	\includegraphics[width=7cm]{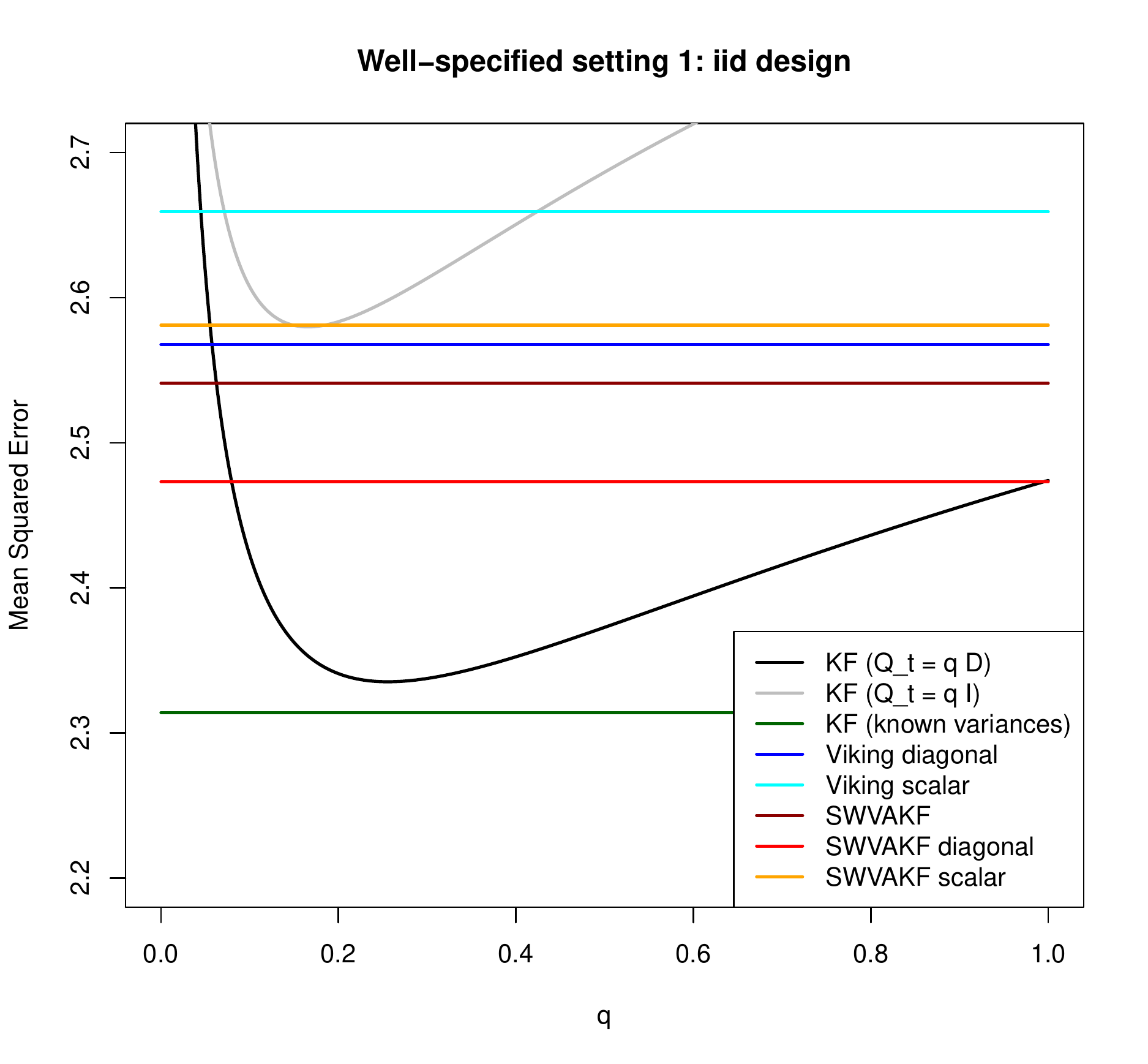}
	\includegraphics[width=7cm]{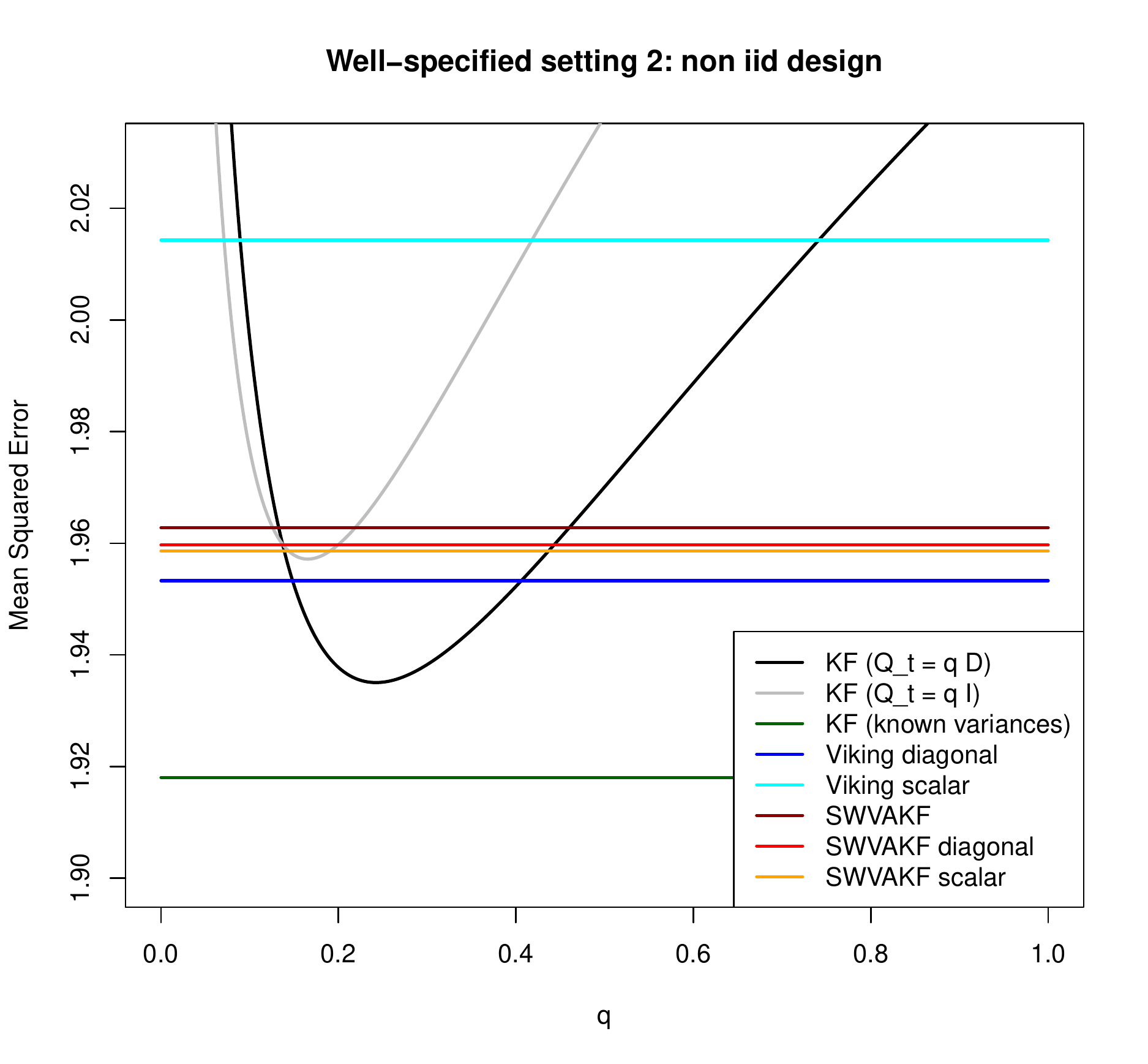}
	\includegraphics[width=7cm]{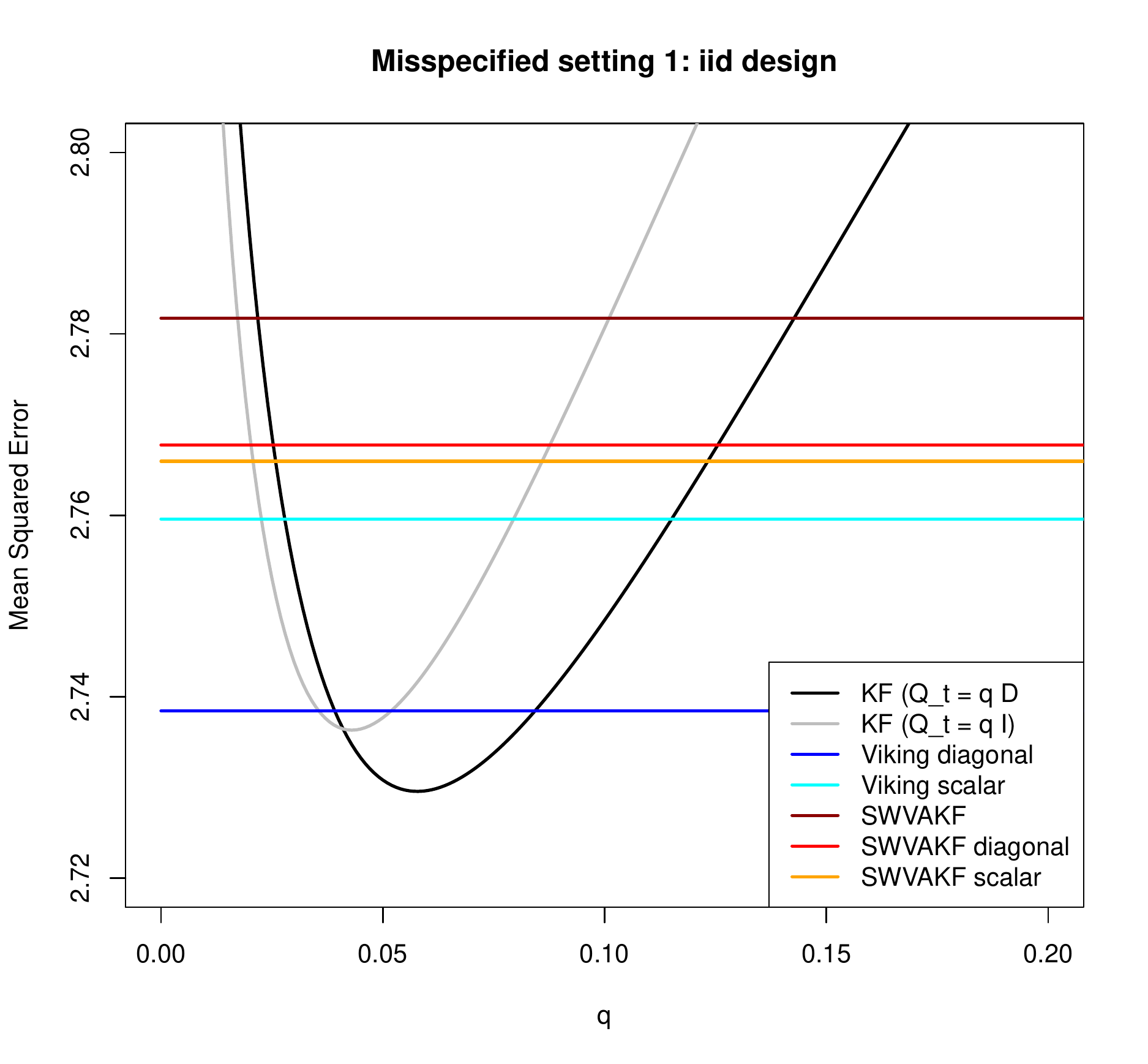}
	\includegraphics[width=7cm]{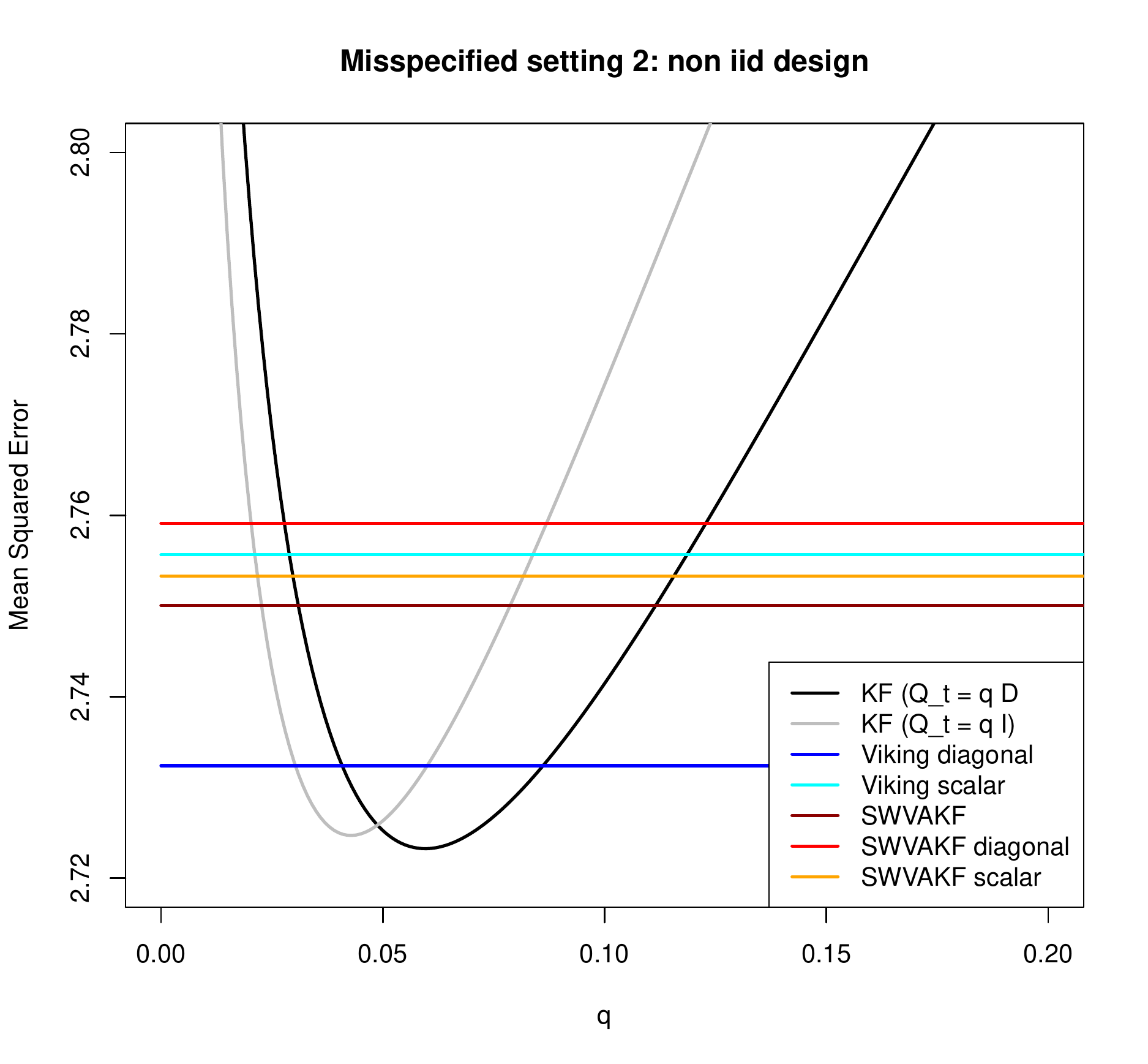}
	\caption{Mean Squared Errors in the four settings introduced in Sections \ref{sec:exp_synws} and \ref{sec:exp_synms}: i.i.d. (left) or non-i.i.d. (right) design, well-specified (top) or misspecified (bottom). We compare Viking to the SWVAKF of \cite{huang2020slide} in the scalar and diagonal settings. For Viking we set $n_{\rm mc}=10$. The oracles to which we compare are the Kalman filter with known variances when they exist (well-specified settings) and two Kalman filters with constant variances: the state noise covariance is either $Q=q\cdot D_{(0,0,1,1,1)}$ or $Q=q\cdot I$ and in both we set the space noise variance to $\sigma^2=1$. We evaluate through the mean squared error on the second half of the experiment in order to not depend too much on the initialization (even if we have same initial expected variances for Viking and SWVAKF).}
	\label{fig:mse_both}
\end{figure*}

\subsection{Impact of $n_{\rm mc}$}
The number of Monte-Carlo samples used at each step to compute $A_t^{-1}$ is a crucial factor of the complexity of Viking. It is therefore necessary to evaluate its impact on the performance in order to reach the best compromise between forecasting and computational efficiencies.
We refer to Figure \ref{fig:impact_mc} for an evaluation of the error with different values of $n_{\rm mc}$. The default value $n_{\rm mc}=10$ seems reasonable.
\begin{figure}
	\centering
	\includegraphics[width=7cm]{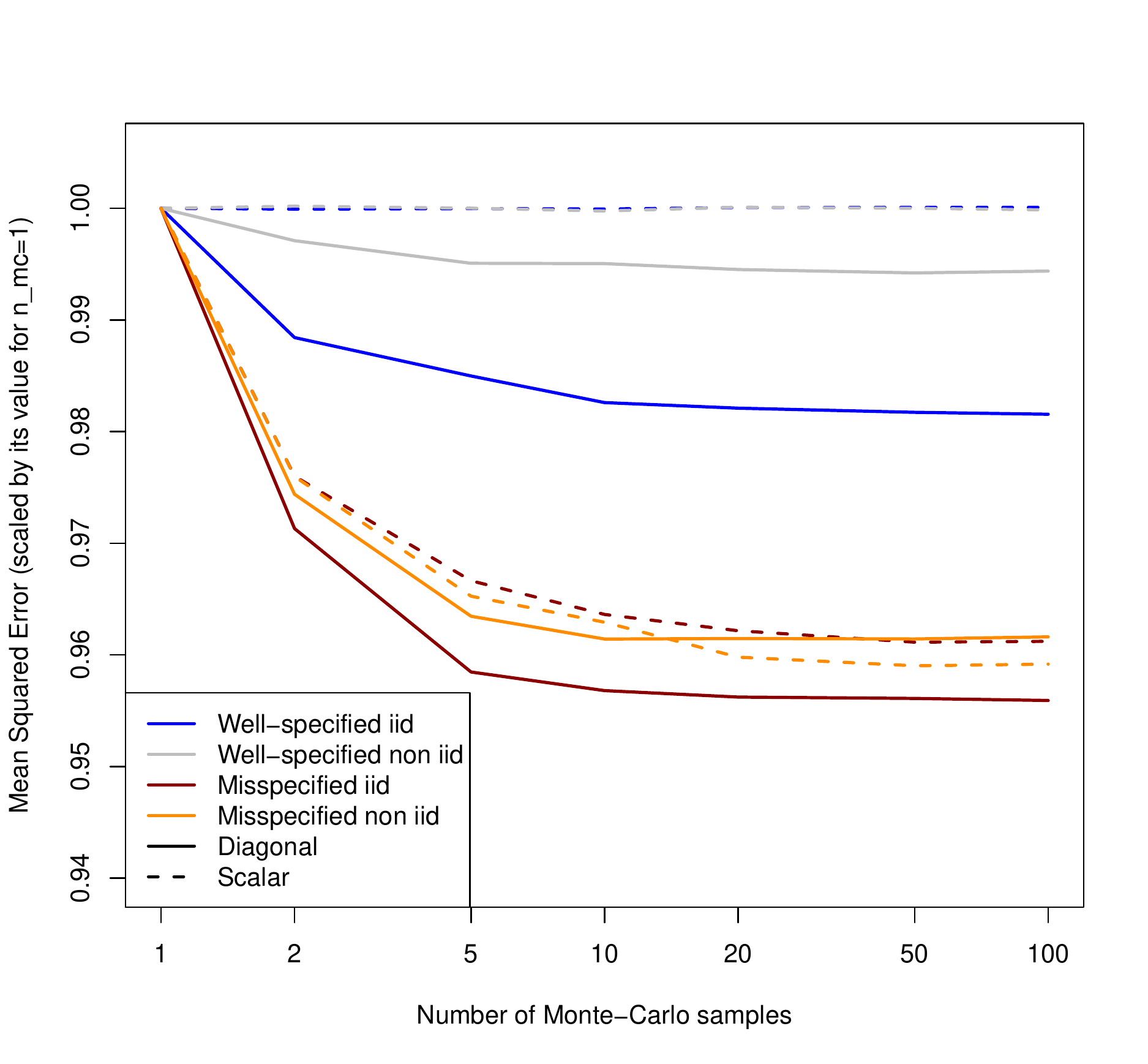}
	\caption{Mean Squared Error of Viking as a function of $n_{\rm mc}$. We scale by the mean squared error of the algorithm with $n_{\rm mc}=1$ in order to fit the different algorithms (diagonal and scalar settings) as well as the different experiments (i.i.d. or non-i.i.d. design, well-specified or misspecified) in the same graph.}
	\label{fig:impact_mc}
\end{figure}

\section{Conclusion}
We have introduced Viking, an algorithm for adaptive time series forecasting relying on state-space models with unknown state and space variances. We derived an augmented latent model, and we apply variational Bayes for the inference. We extend the Kalman filter to uncertain environment. For the additional latent variables, we use approximative steps close to SGVB recursive ones. The prediction performances are better than the state of the art in misspecified settings at the same computational cost. 

The choice of the function applied to the latent variable to obtain the state noise covariance matrix is a perspective of future research. We provide a specific choice leading to promising experimental results on simulations in both well-specified and misspecified settings. However we wrote most of the article considering this function is a parameter of Viking, because we believe other functions may be of interest.

\appendix  

We provide the proofs for all the claims of the article.

\begin{proof}[Proof of Lemma \ref{lemma:kl_expression}]
We start from the expression of \eqref{eq:kl} that we can decompose as follows:
\begin{align*}
	& KL\Big(\mathcal{N}(\hat\theta_{t\mid t},P_{t\mid t})
	\mathcal{N}(\hat a_{t\mid t},s_{t\mid t})
	\mathcal{N}(\hat b_{t\mid t},\Sigma_{t\mid t})\ ||\ 
	p(\cdot\mid \mathcal{F}_t)\Big) \\
	& \quad = \mathbb{E}_{\theta_t\sim\mathcal{N}(\hat\theta_{t\mid t},P_{t\mid t})} [\log \mathcal{N}(\theta_t\mid \hat\theta_{t\mid t},P_{t\mid t})] \\
	& \qquad + 
	\mathbb{E}_{a_t\sim \mathcal{N}(\hat a_{t\mid t},s_{t\mid t})}[\log \mathcal{N}(a_t\mid \hat a_{t\mid t},s_{t\mid t})] \\
	& \qquad + \mathbb{E}_{b_t\sim\mathcal{N}(\hat b_{t\mid t},\Sigma_{t\mid t})}[\log \mathcal{N}(b_t\mid \hat b_{t\mid t},\Sigma_{t\mid t})] \\
	& \qquad - \mathbb{E}_{(\theta_t,a_t,b_t)\sim \mathcal{N}(\hat\theta_{t\mid t},P_{t\mid t})\mathcal{N}(\hat a_{t\mid t},s_{t\mid t})\mathcal{N}(\hat b_{t\mid t},\Sigma_{t\mid t})} \\
	& \qquad\qquad [\log p(\theta_t,a_t,b_t\mid \mathcal{F}_t)]  \,.
\end{align*}
The last term can be split using the factorized form of \eqref{eq:posterior}. We observe that on the one hand,
\begin{align*}
	& \mathbb{E}_{(\theta_t,a_t,b_t)\sim \mathcal{N}(\hat\theta_{t\mid t},P_{t\mid t})\mathcal{N}(\hat a_{t\mid t},s_{t\mid t})\mathcal{N}(\hat b_{t\mid t},\Sigma_{t\mid t})} \\
	& \qquad [\log \mathcal{N}(y_t\mid\theta_t^\top x_t, \exp(a_t))] \\
	& \quad = -\frac12 \log(2\pi)- \frac12 \hat a_{t\mid t} \\
	& \qquad - \frac12 ((y_t-\hat\theta_{t\mid t}^\top x_t)^2 + x_t^\top P_{t\mid t} x_t) \exp(-\hat a_{t\mid t}+\frac12s_{t\mid t}) \,,
\end{align*}
and on the other hand,
\begin{align*}
	& \mathbb{E}_{(\theta_t,a_t,b_t)\sim \mathcal{N}(\hat\theta_{t\mid t},P_{t\mid t})\mathcal{N}(\hat a_{t\mid t},s_{t\mid t})\mathcal{N}(\hat b_{t\mid t},\Sigma_{t\mid t})}  \\
	& \qquad [\log \mathcal{N}(\theta_t\mid K\hat\theta_{t-1\mid t-1}, KP_{t-1\mid t-1}K^\top + f(b_t))] \\
	& \quad = - \frac{d\log(2\pi)}{2} - \frac12 \mathbb{E}_{b_t\sim\mathcal{N}(\hat b_{t\mid t},\Sigma_{t\mid t})}[ \psi_t(b_t) ] \,,
\end{align*}
where $\psi_t$ is defined in the lemma. Combining the last equations with the value of the entropy of gaussian random variables yields the result.
\end{proof}

\begin{proof}[Proof of Theorem \ref{th:optimum_theta}]
Thanks to Lemma \ref{lemma:kl_expression} we have
\begin{align*}
	& KL\Big(\mathcal{N}(\hat\theta_{t\mid t},P_{t\mid t})
	\mathcal{N}(\hat a_{t\mid t},s_{t\mid t})
	\mathcal{N}(\hat b_{t\mid t},\Sigma_{t\mid t})\ ||\ 
	p(\cdot\mid \mathcal{F}_t)\Big) \\
	&\quad = \frac12 \Tr\Big((P_{t\mid t} + (\hat\theta_{t\mid t}-K\hat\theta_{t-1\mid t-1})(\hat\theta_{t\mid t}-K\hat\theta_{t-1\mid t-1})^\top)A_t \Big) \\
	& \qquad + \frac12 ((y_t-\hat\theta_{t\mid t}^\top x_t)^2 + x_t^\top P_{t\mid t} x_t) \exp(-\hat a_{t\mid t} + \frac12 s_{t\mid t}) \\
	& \qquad -\frac12 \log\det P_{t\mid t} + c_{\theta} \,,
\end{align*}
where $c_{\theta}$ is a constant independent of $\hat\theta_{t\mid t},P_{t\mid t}$, and $A_t$ is defined in the theorem. To conclude we write the first order conditions:
\begin{align*}
	& -\frac12 P_{t\mid t}^{-1} + \frac12 \Big(A_t + \frac{x_tx_t^\top}{\exp(\hat a_{t\mid t} - \frac12 s_{t\mid t})} \Big) = 0 \,, \\
	& - \frac{(y_t-\hat\theta_{t\mid t}^{\top} x_t)x_t}{\exp(\hat a_{t\mid t} - \frac12 s_{t\mid t})} + A_t(\hat\theta_{t\mid t}-K\hat\theta_{t-1\mid t-1}) = 0 \,.
\end{align*}
\end{proof}

\begin{proof}[Proof of Proposition \ref{prop:optimum_s}]
Thanks to Lemma \ref{lemma:kl_expression}, we have
\begin{align*}
    & KL\Big(\mathcal{N}(\hat\theta_{t\mid t},P_{t\mid t})
    \mathcal{N}(\hat a_{t\mid t},s_{t\mid t})
    \mathcal{N}(\hat b_{t\mid t},\Sigma_{t\mid t})\ ||\ p(\cdot\mid \mathcal{F}_t)\Big) \\
    & \quad = \frac12 ((y_t-\hat\theta_{t\mid t}^\top x_t)^2 + x_t^\top P_{t\mid t} x_t) e^{-\hat a_{t\mid t} + s_{t\mid t}/2}  \\
    & \qquad + \frac12(s_{t-1\mid t-1}+\rho_a)^{-1} s_{t\mid t} -\frac12 \log(s_{t\mid t}) + c_s \,,
\end{align*}
where $c_s$ is a constant independent of $s_{t\mid t}$. Moreover, if $0\le s_{t\mid t}\le s_{t-1\mid t-1}+\rho_a$ then
\begin{align*}
	e^{s_{t\mid t}/2} \le e^{(s_{t-1\mid t-1}+\rho_a)/2} + \frac12(s_{t\mid t}-(s_{t-1\mid t-1}+\rho_a)) \,.
\end{align*}
The last two equations yield the upper-bound of the proposition. To obtain \eqref{eq:updates} we write the first order condition of optimality:
\begin{align*}
	& \frac14 ((y_t-\hat\theta_{t\mid t}^\top x_t)^2 + x_t^\top P_{t\mid t} x_t) e^{-\hat a_{t\mid t}} - \frac12 s_{t\mid t}^{-1} \\
	& \qquad + \frac12 (s_{t-1\mid t-1}+\rho_a)^{-1} = 0 \,.
\end{align*}
\end{proof}

\begin{proof}[Proof of Proposition \ref{prop:optimum_a}]
Thanks to Lemma \ref{lemma:kl_expression} we have
\begin{align*}
    & KL\Big(\mathcal{N}(\hat\theta_{t\mid t},P_{t\mid t})
    \mathcal{N}(\hat a_{t\mid t},s_{t\mid t})
    \mathcal{N}(\hat b_{t\mid t},\Sigma_{t\mid t})\ ||\ p(\cdot\mid \mathcal{F}_t)\Big) \\
    & \quad \le \frac12 ((y_t-\hat\theta_{t\mid t}^\top x_t)^2 + x_t^\top P_{t\mid t} x_t) e^{-\hat a_{t\mid t} + s_{t \mid t}/2} \\
    & \qquad +  \frac12 (s_{t-1\mid t-1}+\rho_a)^{-1} (\hat a_{t\mid t}-\hat a_{t-1\mid t-1})^2 + \frac12 \hat a_{t\mid t} + c_a \,,
\end{align*}
with $c_a$ a constant independent of $\hat a_{t\mid t}$. Moreover, if $\hat a_{t\mid t}\in[\hat a_{t-1\mid t-1}- M_a,\hat a_{t-1\mid t-1}+ M_a]$ we have the following upper-bound:
\begin{align*}
	& e^{-\hat a_{t\mid t}} \le e^{-\hat a_{t-1\mid t-1}}\Big(1 - (\hat a_{t\mid t} - \hat a_{t-1\mid t-1}) \\
	& \qquad\qquad\qquad\qquad\qquad + \frac{e^{M_a}}{2}(\hat a_{t\mid t} - \hat a_{t-1\mid t-1})^2\Big) \,.
\end{align*}
The last two equations yield the upper-bound of the proposition.
To obtain \eqref{eq:updatea} we write the first-order condition:
\begin{align*}
	& \frac{1}{s_{t-1\mid t-1}+\rho_a}(\hat a_{t\mid t}-\hat a_{t-1\mid t-1}) + \frac12 \\
	& \qquad + \frac12((y_t-\hat\theta_{t\mid t}^\top x_t)^2 + x_t^\top P_{t\mid t} x_t) e^{-\hat a_{t-1\mid t-1} + s_{t\mid t}/2} \\
	& \qquad\qquad \Big(-1+e^{M_a}(\hat a_{t\mid t}-\hat a_{t-1\mid t-1})\Big) = 0 \,,
\end{align*}
\end{proof}

To prove Propositions \ref{prop:upper} and \ref{prop:optimum_b} we first compute the first and second derivatives of $\psi_t$ for the scalar and diagonal settings: 
\begin{lemma}\label{lemma:derivatives}
Let $C_t=KP_{t-1\mid t-1}K^\top + f(b_t)$ and $B_t = P_{t\mid t} + (\hat\theta_{t\mid t}-K\hat\theta_{t-1\mid t-1})(\hat\theta_{t\mid t}-K\hat\theta_{t-1\mid t-1})^\top$.
\begin{itemize}
\item
If $f(\cdot)=\phi(\cdot)I$ then for any $b_t$, we have
\begin{align*}
	& \psi_t'(b_t) = \Tr(C_t^{-1}(I - B_tC_t^{-1}))\phi'(b_t) \,, \\
	& \psi_t''(b_t) = \Tr(C_t^{-1}(I - B_tC_t^{-1}))\phi''(b_t) \\
	& \qquad\qquad + 2 \Tr(C_t^{-2}(B_tC_t^{-1} - I/2)) \phi'(b_t)^2 \,.
\end{align*}
\item
If $f(\cdot)=D_{\phi(\cdot)}$ then for any $b_t$, we have
\begin{align*}
	& \frac{\partial \psi_t}{\partial b_t} = \Delta_{C_t^{-1}(I - B_tC_t^{-1})}\odot \phi'(b_t) \,, \\
	& \frac{\partial^2 \psi_t}{\partial b_t^2} = C_t^{-1}(I - B_tC_t^{-1}) D_{\phi''(b_t)} \odot I  \\
	& \qquad\qquad + 2 C_t^{-1}(B_tC_t^{-1} - I/2) \odot C_t^{-1} \odot \phi'(b_t)\phi'(b_t)^\top \,,
\end{align*}
where $\odot$ is the Hadamard (pointwise) product.
\end{itemize}
\end{lemma}

\begin{proof}
\begin{itemize}
\item
In the scalar setting we recall that
\begin{align*}
	\psi_t(b) & = \log\det(KP_{t-1\mid t-1}K^\top + \phi(b) I) \\
	& \qquad + \Tr(B_t(KP_{t-1\mid t-1}K^\top + \phi(b) I)^{-1}) \,.
\end{align*}
We denote by $\log$ and $\exp$ the univariate logarithm and exponential and by $\Log$ the matrix logarithm. Note that if $A\succ 0$, it holds $\det A=\exp\Tr(\Log A)$. We define $C_t=KP_{t-1\mid t-1}K^\top + \phi(b_t)I$ and we obtain:
\begin{align*}
	& \log\det(KP_{t-1\mid t-1}K^\top + \phi(b)I) - \Tr\Log(C_t) \\
	& \quad = \Tr\Log(KP_{t-1\mid t-1}K^\top + \phi(b)I) - \Tr\Log(C_t) \\
	& \quad = \Tr\Log\Big(I + (\phi(b)-\phi(b_t))C_t^{-1}\Big) \\
	& \quad = \Tr\Big(\Big(\phi'(b_t)(b-b_t)+\frac12 \phi''(b_t) (b-b_t)^2\Big)C_t^{-1} \\
	& \qquad\qquad - \frac12 \big(\phi'(b_t)(b-b_t)C_t^{-1}\big)^2 + o((b - b_t)^2) \Big) \,.
\end{align*}
The last line follows from the series expansion of the Logarithm.
We apply another series expansion for the second term of $\psi_t$: we have
\begin{align*}
	& \Tr(B_t (KP_{t-1\mid t-1}K^\top + \phi(b) I)^{-1}) \\
	& \quad = \Tr\Big(B_t C_t^{-1} \Big(I + (\phi(b)-\phi(b_t))C_t^{-1} \Big)^{-1} \Big) \\
	& \quad = \Tr\Big(B_t C_t^{-1} \\
	& \qquad\qquad \Big(I - \Big(\phi'(b_t)(b-b_t)+\frac12 \phi''(b_t) (b-b_t)^2\Big)C_t^{-1} \\
	& \qquad\qquad\quad + \big(\phi'(b_t)(b-b_t)C_t^{-1}\big)^2 + o((b - b_t)^2) \Big) \Big) \,.
\end{align*}

Summing the last two equations, and using the identity $\Tr(AB)=\Tr(BA)$, we can identify the first and second derivatives of $\psi_t$.

\item
We develop a similar argument in the diagonal setting:
\begin{align*}
	\psi_t(b) & = \log\det(KP_{t-1\mid t-1}K^\top+ D_{\phi(b)}) \\
	& \quad + \Tr(B_t (KP_{t-1\mid t-1}K^\top + D_{\phi(b)})^{-1}) \,,
\end{align*}
then we apply the series expansion of the Logarithm:
\begin{align*}
	& \log\det(KP_{t-1\mid t-1}K^\top + D_{\phi(b)}) - \Tr\Log(C_t) \\
	& \quad = \Tr\Log\Big(I + D_{\phi(b) - \phi(b_t)} C_t^{-1}\Big) \\
	& \quad = \Tr\Big(D_{\phi'(b_t)(b-b_t)+\frac12 \phi''(b_t) (b-b_t)^2}C_t^{-1} \\
	& \qquad\qquad - \frac12 (D_{\phi'(b_t)(b-b_t)}C_t^{-1})^2 + o(\|b - b_t\|^2) \Big) \,,
\end{align*}
where $C_t=KP_{t-1\mid t-1}K^\top + D_{\phi(b_t)}$ and $\phi'(b_t),\phi''(b_t)$ denote the coefficient-wise application of the first and second derivatives of $\phi$ to the vector $b_t$.
We apply another series expansion for the second term of $\psi_t$:
\begin{align*}
	& \Tr\Big(B_t (KP_{t-1\mid t-1}K^\top + D_{\phi(b)})^{-1}\Big) \\
	& \quad = \Tr\Big(B_t C_t^{-1} \Big(I + D_{\phi(b)-\phi(b_t)}C_t^{-1} \Big)^{-1} \Big) \\
	& \quad = \Tr\Big(B_t C_t^{-1} \\
	& \qquad\qquad \Big(I - D_{\phi'(b_t)(b-b_t)+\frac12 \phi''(b_t) (b-b_t)^2}C_t^{-1} \\
	& \qquad\qquad\quad + (D_{\phi'(b_t)(b-b_t)}C_t^{-1})^2 + o(\|b - b_t\|^2) \Big) \Big) \,.
\end{align*}

Summing the last two equations we obtain
\begin{align*}
	\psi_t(b) & = \Tr\Log(C_t) + \Tr(B_tC_t^{-1}) \\
	& \quad + \Tr\Big(C_t^{-1}(I - B_tC_t^{-1}) \\
	& \qquad\qquad\qquad D_{\phi'(b_t)(b-b_t)+\frac12 \phi''(b_t) (b-b_t)^2}\Big) \\
	& \quad + \Tr\Big( C_t^{-1}(B_tC_t^{-1} - I/2) \\
	& \qquad\qquad\qquad D_{\phi'(b_t)(b-b_t)} C_t^{-1} D_{\phi'(b_t)(b-b_t)}\Big) \\
	& \quad + o(\|b - b_t\|^2) \,.
\end{align*}
Then we use the identity $Tr(AD_vBD_v)=v^\top (A \odot B^\top)v$. We have
\begin{align*}
	& \psi_t(b) = \Tr\Log(C_t) + \Tr(B_tC_t^{-1}) \\
	& \quad + \frac12 (b-b_t)^\top \Big( C_t^{-1}(I - B_tC_t^{-1}) D_{\phi''(b_t)} \odot I  \\ 
	& \qquad\qquad\qquad\qquad + 2 C_t^{-1}(B_tC_t^{-1} - I/2) \odot \\
	& \qquad\qquad\qquad\qquad\qquad C_t^{-1}\odot \phi'(b_t)\phi'(b_t)^\top\Big) (b-b_t) \\
	& \quad + (\Delta_{C_t^{-1}(I - B_tC_t^{-1})}\odot \phi'(b_t))^\top (b-b_t) + o(\|b - b_t\|^2) \,.
\end{align*}
Thus we can identify the first and second derivatives of $\psi_t$.
\end{itemize}
\end{proof}

\begin{proof}[Proof of Proposition \ref{prop:upper}]
As long as $f(\hat b_{t-1\mid t-1})\succ 0$ we know that $f$ is twice differentiable in $\hat b_{t-1\mid t-1}$ and the local upper-bound property of Proposition \ref{prop:upper} holds if $\frac{\partial^2 \psi_t}{\partial b_t^2}|_{\substack{\hat b_{t-1\mid t-1}}} \prec H_t$. We bound the expressions obtained in Lemma \ref{lemma:derivatives}.
\begin{itemize}
\item
In the scalar setting,
\begin{multline*}
	\psi_t''(\hat b_{t-1\mid t-1}) = \Tr(C_t^{-1}(I - B_tC_t^{-1}))\phi''(\hat b_{t-1\mid t-1}) \\
	+ 2 \Tr(C_t^{-2}(B_tC_t^{-1} - I/2)) \phi'(\hat b_{t-1\mid t-1})^2 \,.
\end{multline*}
Furthermore, $C_t\succ 0$ thus $C_t^{-1}\succ 0$, $\Tr(C_t^{-1})>0$, and $\Tr(C_t^{-2})>0$. $\phi''(\hat b_{t-1\mid t-1})=-1/(1+\hat b_{t-1\mid t-1})^2<0$ and $\phi'(\hat b_{t-1\mid t-1})^2>0$, therefore we obtain
\begin{align*}
	\psi_t''(\hat b_{t-1\mid t-1}) & < -\Tr(C_t^{-1}B_tC_t^{-1})\phi''(\hat b_{t-1\mid t-1}) \\
	& \quad + 2 \Tr(C_t^{-2}B_tC_t^{-1}) \phi'(\hat b_{t-1\mid t-1})^2 \,.
\end{align*}
\item
In the diagonal setting,
\begin{align*}
	\frac{\partial^2 \psi_t}{\partial b_t^2}\Big|_{\substack{\hat b_{t-1\mid t-1}}} & = C_t^{-1}(I - B_tC_t^{-1}) D_{\phi''(\hat b_{t-1\mid t-1})} \odot I  \\
	& \quad + 2 C_t^{-1}(B_tC_t^{-1} - I/2) \odot C_t^{-1} \odot \\
	& \qquad\qquad\qquad \phi'(\hat b_{t-1\mid t-1})\phi'(\hat b_{t-1\mid t-1})^\top \,.
\end{align*}
Similarly we have $C_t^{-1}\succ 0, D_{\phi''(\hat b_{t-1\mid t-1})}\prec 0$ and as diagonal coefficients of $C_t^{-1}$ are positive, it yields $(C_t^{-1}D_{\phi''(\hat b_{t-1\mid t-1})})\odot I \prec 0$.

Moreover $\phi'(\hat b_{t-1\mid t-1})\phi'(\hat b_{t-1\mid t-1})^\top\succ 0$, and we can apply Schur product theorem: $C_t^{-1} \odot C_t^{-1} \odot \phi'(\hat b_{t-1\mid t-1})\phi'(\hat b_{t-1\mid t-1})^\top\succ 0$.
Eventually:
\begin{align*}
	& \frac{\partial^2 \psi_t}{\partial b_t^2}\Big|_{\substack{\hat b_{t-1\mid t-1}}} \prec -C_t^{-1}B_tC_t^{-1} D_{\phi''(\hat b_{t-1\mid t-1})} \odot I  \\
	& \quad + 2 C_t^{-1}B_tC_t^{-1}  \odot C_t^{-1} \odot \phi'(\hat b_{t-1\mid t-1})\phi'(\hat b_{t-1\mid t-1})^\top \,.
\end{align*}
\end{itemize}
\end{proof}

\begin{proof}[Proof of Proposition \ref{prop:optimum_b}]
Thanks to Lemma \ref{lemma:kl_expression} we have:
\begin{align*}
	& KL\Big(\mathcal{N}(\hat\theta_{t\mid t},P_{t\mid t})
	\mathcal{N}(\hat a_{t\mid t},s_{t\mid t})
	\mathcal{N}(\hat b_{t\mid t},\Sigma_{t\mid t})\ ||\ p(\cdot\mid \mathcal{F}_t)\Big) \\
	& \quad = -\frac12 \log\det\Sigma_{t\mid t} + \frac12 \mathbb{E}_{b_t\sim\mathcal{N}(\hat b_{t\mid t},\Sigma_{t\mid t})} [\psi_t(b_t)] \\
	& \qquad + \frac12 \Tr\Big( (\Sigma_{t\mid t} + (\hat b_{t\mid t}-\hat b_{t-1\mid t-1})(\hat b_{t\mid t}-\hat b_{t-1\mid t-1})^\top ) \\
	& \qquad\qquad\qquad\qquad\qquad (\Sigma_{t-1\mid t-1} + \rho_b I)^{-1} \Big) + c_b \,,
\end{align*}
where $c_b$ is a constant independent of $\hat b_{t\mid t}, \Sigma_{t\mid t}$.
Combining the last equation and Proposition \ref{prop:upper}, then using the first two moments of the gaussian distribution we obtain:
\begin{align*}
	& KL\Big(\mathcal{N}(\hat\theta_{t\mid t},P_{t\mid t})
	\mathcal{N}(\hat a_{t\mid t},s_{t\mid t})
	\mathcal{N}(\hat b_{t\mid t},\Sigma_{t\mid t})\ ||\ 
	p(\cdot\mid \mathcal{F}_t)\Big) \\
	& \quad \le -\frac12 \log\det\Sigma_{t\mid t} + \frac12 \psi_t(\hat b_{t-1\mid t-1}) \\
	& \qquad + \frac12 \frac{\partial \psi_t}{\partial b_t}\Big|_{\substack{\hat b_{t-1\mid t-1}}}^\top (\hat b_{t\mid t}-\hat b_{t-1\mid t-1}) \\
	& \qquad + \frac14 \Tr\Big(H_t (\Sigma_{t\mid t} + (\hat b_{t\mid t} - \hat b_{t-1\mid t -1})(\hat b_{t\mid t} - \hat b_{t-1\mid t -1})^\top)\Big) \\
	& \qquad + \frac12 \Tr\Big( (\Sigma_{t\mid t} + (\hat b_{t\mid t}-\hat b_{t-1\mid t-1})(\hat b_{t\mid t}-\hat b_{t-1\mid t-1})^\top ) \\
	& \qquad\qquad\qquad\qquad\qquad (\Sigma_{t-1\mid t-1} + \rho_b I)^{-1} \Big) + c_b \,.
\end{align*}
This yields the upper-bound of Proposition \ref{prop:optimum_b}. The recursive updates follow from the first order conditions:
\begin{align*}
	& - \frac12 \Sigma_{t\mid t}^{-1} + \frac12 \Big((\Sigma_{t-1\mid t-1} + \rho_bI)^{-1} + \frac12 H_t \Big)  = 0 \,, \\
	& \Big((\Sigma_{t-1\mid t-1} + \rho_b I)^{-1} + \frac12 H_t \Big) (\hat b_{t\mid t}-\hat b_{t-1\mid t-1}) \\
	& \qquad + \frac12 \frac{\partial \psi_t}{\partial b_t}\Big|_{\substack{\hat b_{t-1\mid t-1}}} = 0 \,.
\end{align*}
\end{proof}


\bibliographystyle{IEEEtran}
\bibliography{IEEEabrv,mybib}

%








\end{document}